\documentclass{article} % For LaTeX2e
\usepackage{iclr2026_conference,times}

\usepackage{amsmath,amsfonts,bm}

\def\eqref#1{equation~\ref{#1}}
\def\1{\bm{1}}

\DeclareMathAlphabet{\mathsfit}{\encodingdefault}{\sfdefault}{m}{sl}
\SetMathAlphabet{\mathsfit}{bold}{\encodingdefault}{\sfdefault}{bx}{n}

\usepackage{hyperref}
\usepackage{url}

\usepackage[utf8]{inputenc} % allow utf-8 input
\usepackage[T1]{fontenc}    % use 8-bit T1 fonts
\usepackage{hyperref}       % hyperlinks
\usepackage{url}            % simple URL typesetting
\usepackage{booktabs}       % professional-quality tables
\usepackage{amsfonts}       % blackboard math symbols
\usepackage{nicefrac}       % compact symbols for 1/2, etc.
\usepackage{microtype}      % microtypography
\usepackage{xcolor}         % colors
\usepackage{graphicx}
\usepackage{amsmath}
\usepackage{amsthm}
\newtheorem{theorem}{Theorem}
\newtheorem{lemma}{Lemma}
\newtheorem{definition}{Definition}

\usepackage{comment}

\title{Occam Gradient Descent}

\author{%
  B.N. Kausik\thanks{https://sites.google.com/view/bnkausik/profile \\ Thanks to Jay Jawahar, Ravi Kannan, Prasad Tadepalli and Leslie Valiant for comments and suggestions.} \\
  Independent\\
  \texttt{bnkausik@gmail.com} \\
  }

\iclrfinalcopy % Uncomment for camera-ready version, but NOT for submission.
\begin{document}

\maketitle

\begin{abstract}
Deep learning neural network models must be large enough to adapt to their problem domain, while small enough to avoid overfitting training data during gradient descent.  To balance these competing demands, over-provisioned deep learning models such as transformers are trained for a single epoch on large data sets, and hence inefficient with both computing resources and training data.  In response to these inefficiencies, we derive a provably good algorithm that can combine any training and pruning methods to simultaneously optimize efficiency and accuracy,  identifying conditions that resist overfitting and reduce model size while outperforming the underlying training algorithm. We then use the algorithm to combine gradient descent with magnitude pruning into "Occam Gradient Descent." With respect to loss, compute and model size  (a) on image classification benchmarks, linear and convolutional neural networks trained with Occam Gradient Descent outperform traditional gradient descent with or without post-train pruning; (b) on a range of tabular data classification tasks, neural networks trained with Occam Gradient Descent outperform traditional gradient descent, as well as Random Forests; (c) on natural language transformers, Occam Gradient Descent outperforms traditional gradient descent.
\end{abstract}

\section{Introduction}

Deep learning models are artificial neural networks often with hundreds of billions of parameters, e.g.,\citet{bmrskdoa20}, \citet{rbcmhsoi21}, \citet{spnlrcoc22}, \citet{tdhskcol22}.  However, trained models are sparse in that most of the parameters are negligible, raising the question as to whether the models really need to be large to perform, e.g., \citet{kmhbccgrwa20} and \citet{hbmbcros22}  in the context of Large Language Models (LLMs).  \citet{k24} suggests that LLMs are vastly overprovisioned compared to the theoretical estimated dimensionality of the training data.  While overprovisioned models can adapt well to the problem domain, they are prone to overfitting and poor generalization, \citet{clot21}.  As a result, large models are typically trained for just a single epoch, \cite{xfzzy23}.   

Our results are related to several categories of work in the literature.   Firstly, network pruning, e.g., \cite{lds89} \citet{hptd15}, \citet{hassibi1993optimal}, \citet{hptd15},  \citet{lszhd18}, \citet{bgfg20},\citet{slbk23},  \citet{habdp21}  and \citet{fa23}, which set to zero some of the parameters of a trained network in order to reduce model size with minimal loss of accuracy.  Secondly, knowledge distillation, e.g., \citet{hvd15}, \citet{ccyhc17} and \citet{amyma17}, which seek a smaller network that mimics a larger network with minimal loss of accuracy, and the application of distillation to regularization, e.g. \citet{ytlwf20}, \citet{gm21}.   Thirdly, learning theory, \citet{v84}, \citet{n89}, \citet{h92}, \citet{sb14}, \citet{behw87}, \citet{bp90},  and \citet{n93}.  Fourthly, regularization, e.g., \citet{t96}, \citet{kgc17}, which invokes Occam’s Razor to minimize the norm of the weights of a neural network by including the norm as an additive term in the loss function for gradient descent.   While regularization has empirical benefits, it does not reduce the model size, and including the norm in the loss function creates an ad hoc tradeoff between the training loss and the norm.

Building upon prior work, our Occam Pruning algorithm (named after Occam’s Razor: “The simplest explanation is most likely correct”) combines any training and pruning methods to simultaneously optimize efficiency and accuracy, identifying conditions for (a) monotonic convergence that resists overfitting (b) model size decreasing exponentially in the loss (c) outperforming the underlying training algorithm. We then use the algorithm to combine gradient descent with magnitude pruning into "Occam Gradient Descent,"  which works on any neural network without modifications or limitations such as random graphs or Gumbel softmax, e.g.,\citet{mmsngl18}, \citet{ztz23}.  Furthermore, our algorithm adaptively prunes a network to outperform the original network, improving on the "lottery ticket hypothesis" of \citet{fc18}, who hypothesize an ad hoc pruning method for comparable performance. We also note that in contrast to the large body of work on Neural Architecture Search, e.g. as surveyed in \citet{white2023neural}, our results are focused on training a given neural network.  

With respect to loss, compute and model size, our experiments show (a) on the MNIST and CIFAR10 image classification benchmarks, linear and convolutional neural networks trained with Occam Gradient Descent outperform traditional gradient descent with or without post-train pruning; (b) on a range of tabular data classification tasks, neural networks trained with Occam Gradient Descent outperform traditional gradient descent, as well as Random Forests, e.g.,\citet{sutton2005classification}, \cite{bs16}; (c) on natural language transformers, Occam Gradient Descent outperforms traditional gradient descent. While the experiments in this paper use magnitude pruning with the Adam and AdamW optimizers for traditional gradient descent, we expect similar performance on other variants per our theoretical results.

\section{Theoretical Results}
\label{theoretical_results}

Consider functions of the form $f:X\rightarrow [k]$, where $X$ is the domain and $[k]=\{1,2,...k\}$ is the range of $k$ labels. A neural network computes a space of functions $F:W \times [k]$, where $W$ is the space of the weights of the network.  For a specific choice of weights $w \in W $, the function $f:X\rightarrow[k]$ computed by the network is denoted as $f(x)=F(w,x)$, for $x\in X$. 

For function $f:X\rightarrow [k]$  and a probability distribution $P$ on $ X\times [k]$, the discrete loss of $f$ with respect to $P$ is the probability that $f$ is incorrect, i.e.,  
\begin{equation}
L(f,P)= P\{(x,y) : f(x)\neq y\} 
\end{equation}

\textbf{Neural Network Training Problem: }Given a collection of training samples $S=\{(x_i,y_i)\}$ drawn on a distribution $P$, compute $w \in W$ such that $f(x)=F(w,x)$ minimizes $L(f,P)$.

Gradient descent is the established method to solve the above problem.  Starting with random initial values for the weights, the training loss $L(f,S)$ is reduced along its steepest gradient on the weights, iterating over the training samples. Each pass across the set of training samples is called an epoch. However, if the training algorithm is run for multiple epochs on the training set, the weights are optimized for a distribution that favors the training samples rather than the natural distribution.  This is commonly referred to as overfitting, and leads to poor generalization and low accuracy on test data. To avoid overfitting, neural network models are over-provisioned but trained for just one epoch on the training data, resulting in inefficient use of both computing resources and training data.  Towards a theoretical analysis of the above, we examine the relevant learning theoretic results as surveyed in \citet{sb14}.

\begin{definition} Given $\epsilon$ and $\delta$ in (0,1), and samples drawn on probability distribution $P$, an agnostic learning algorithm finds $f\in F$ with confidence at least $(1-\delta)$ such that $L(f,P) = \min_{h\in F}L(h,P)+ \epsilon$
\end{definition}
There are several measures for the sample complexity of a space of functions that can be used to bound , such as the Vapnik Chervonenkis dimension, the Generalized dimension and Rademacher complexity.  Since we need bounds on multi-class functions, we use the notion of the Generalized dimension, \citet{n89}, also known as the Natarajan dimension, \citet{sb14}.   For brevity, we will simply refer to it as the dimension.

\begin{definition} (Generalized shattering) A set $C \subset X$ is \textit{shattered} by $F$ if there exist two functions $f_0, f_1$ in $F$ such that for every $x \in C$, $f_0(x) \neq f_1 (x)$; and for every $B \subset C$, there exists $f_2 \in F$, such that for all $x \in B, f_2(x) =f_0(x)$ , and for all $x \in (C-B)$, $f_2(x) = f_1(x)$.
\end{definition}

\begin{definition} The dimension of a space of functions $F$ is the size of the largest set shattered by it, and is denoted by $dim(F)$.
\end{definition}

Intuitively, the dimension of a space of functions is a measure of the richness of the space, i.e. the number of degrees of freedom across the functions in the space.   The following theorem is adapted from \citet{sb14}, and proved therein, based on \citet{h92}.
\begin{theorem}
For fixed $k$ and $\delta$, the error in any agnostic learning algorithm for a space of functions $F$ from $m$ samples is 
$\Theta [ (dim(F)/m)^{0.5}]$
\end{theorem}
Therefore for fixed $k,\delta$, the expected test loss $L(f,P)$ of an agnostic learning algorithm is of the form 
\begin{equation}\label{agnostic_theta}
     L(f,P) = \min_{h\in F}L(h,P)+ \Theta[(dim(F)/m)^{0.5}]
\end{equation}
Equation \ref{agnostic_theta} applies to a neural network. $L(f,P)$ is the expected test loss, $min_{h\in F}L(h,P)$ is the approximation error, while the $[dim(F)/m]$ term is the estimation error.  A larger network with more trainable weights reduces the approximation error.   But the dimensionality of the network scales with the number of weights, and hence the estimation error increases with the size of the network.  Fig. 1 shows an oversized off-the shelf linear network (https://www.tensorflow.org/datasets/keras\_example) for the MNIST dataset.
\begin{figure}
    \centering
    \includegraphics[width=0.75\linewidth]{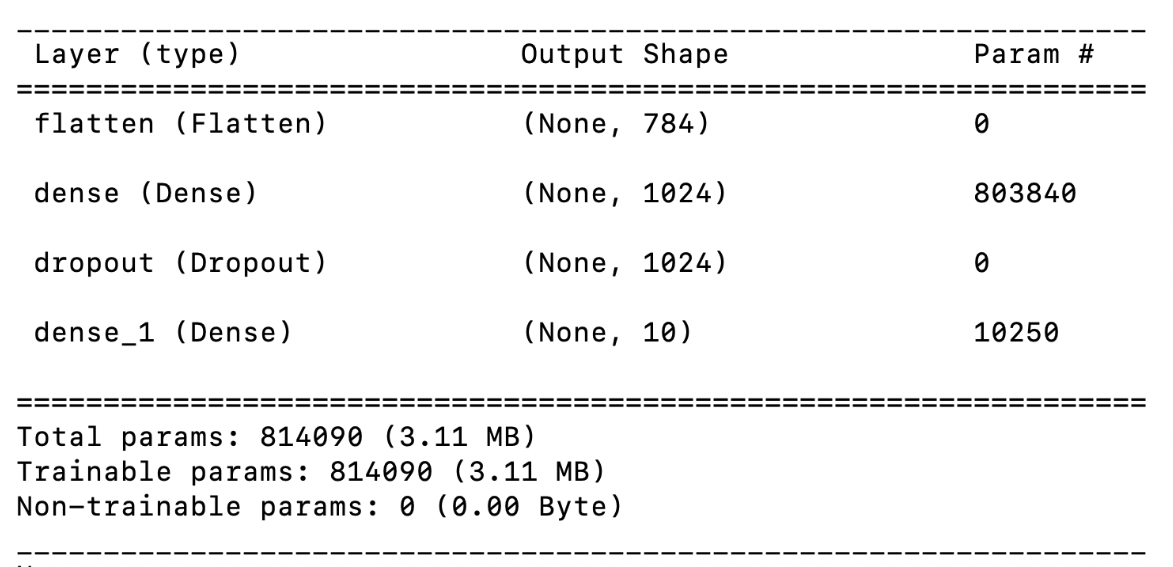}
    \caption{Oversized off-the shelf linear network for MNIST}
    \label{fig:MNIST_network}
\end{figure}
\begin{figure}
    \centering
    \includegraphics[width=0.45\linewidth]{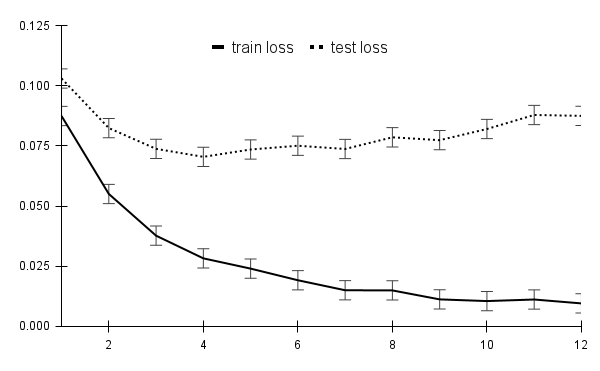}
    \includegraphics[width=0.45\linewidth]{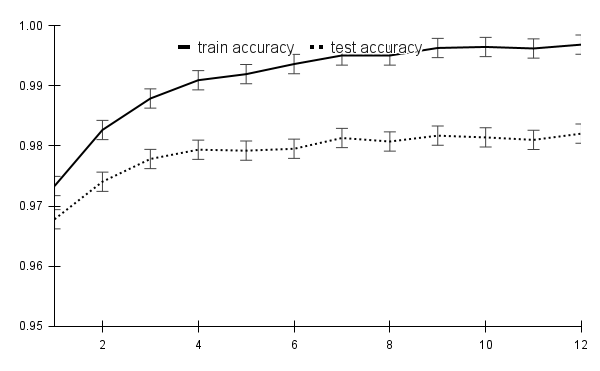}
    \caption{Gradient descent loss \& accuracy; MNIST; average of ten runs.}
    \label{fig:gd_loss}
\end{figure}

Fig. \ref{fig:gd_loss} shows\footnote{Throughout this paper, error bars are twice the maximum standard error of the mean across each plot, indicative of variability across the training runs.} the training and test cross-entropy loss and accuracy for the network across epochs during gradient descent training. Increasing the number of training epochs improves the training loss and accuracy, but the test loss improves for the first few epochs and then degrades due to overfitting.    On extremely large networks, the test loss may degrade after one epoch, resulting in poor utilization of the training samples, \citet{xfzzy23}.

To improve upon the above, we first analyze overfitting in the context of Equation \ref{agnostic_theta}.  Gradient descent training (a) initializes the weights in the network to small random values and then (b) operating on the training samples, amplifies weights that improve the training loss most.  Additional training epochs repeat step (b) on the training samples.  Fig. \ref{fig:conceptual} is a conceptual visualization of gradient descent. The ball is the space $F$ of all functions computable by the neural network. There are two points of interest on the boundary of the ball. First, the function computed by the neural network that minimizes the loss on the natural distribution, i.e. the  minimum-loss hypothesis $h_P={\arg}{\min}_{h\in F}L(h,P)$. Second, the function computed by the neural network instance that minimizes the loss on the training samples, i.e. $h_S={\arg}{\min}_{h\in F}L(h,S)$.   In other words, $h_P$ minimizes the test loss, while $h_S$ minimizes the training loss. The weights of the network are randomly initialized to compute a function $f_0 \in F$.  Thereafter, gradient descent iterates over the training samples and adjusts the weights to reduce the training loss so that the network computes functions $f_1, f_2, f_3...f_i$ after successive epochs.  During the first epoch, the function computed by the network moves closer to $h_P$.  In later epochs, the training samples are repeated, thereby pulling the network towards $h_S$.  In brief, the training loss, which is represented by the distance between $f_i$ and $h_S$, decreases across epochs. The test loss, which is represented by the distance between $f_i$ and $h_P$, initially decreases and then increases, causing the overfitting phenomenon of Fig. \ref{fig:gd_loss}.
\begin{figure}
    \centering
    \includegraphics[width=0.75\linewidth]{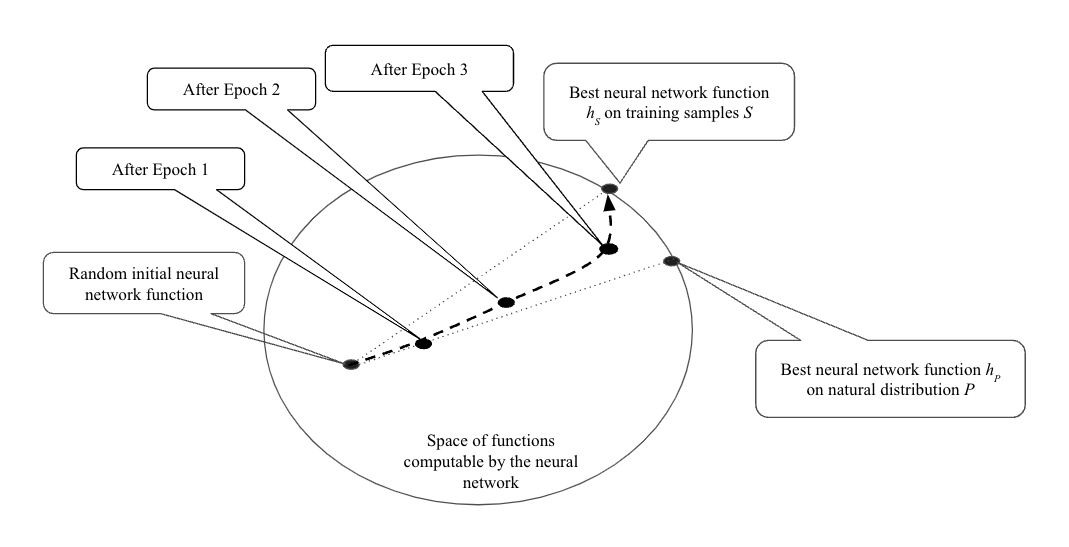}
    \caption{Conceptual visualization of gradient descent across epochs}
    \label{fig:conceptual}
\end{figure}

We seek to improve gradient descent so that we can train the network for multiple epochs without overfitting, thereby extracting more of the information in the training samples.  To that end, we first define a well-conditioned pruning algorithm as follows.   

\begin{definition}
    A well-conditioned pruning algorithm takes as input a neural network computing a function $f$ in a space of functions $F$ and a pruning step size $\lambda \in [0,\hat{\lambda}]\subseteq[0,1]$, and produces as output a reduced network  $G\subset F$, such that (a) the dimension of the space is reduced by a fraction $\lambda$, i.e., 
\begin{align}\label{prune_dim}
 dim(G) \leq (1-\lambda) dim(F)\
\end{align}
and (b) satisfies a Lipschitz condition in that there exists $ c\geq 0$ such that
\begin{align}\label{prune_error}
\min_{h\in G} L(h,S) -\min_{h\in F} L(h,S) \leq  c \lambda | L(f,S)   - \min_{h\in F} L(h,S)|
\end{align}
\end{definition}

Equation \ref{prune_error} bounds the error introduced by pruning via the Lipschitz contstant $c$ over the interval $[0,\hat{\lambda}]$ . The pruning algorithm may use any means to reduce the space of functions computed by the network, e.g., remove, freeze or quantize weights. With the foregoing in hand, we interleave training and pruning to improve performance and efficiency.  Specifically, starting with $F=F_0$, we alternate training and pruning such that $F_0 \supset F_1 \supset F_2...\supset F_i$, where $F_i$ is the space of functions computed by the network after pruning at the $i^{th}$ epoch, and $f_i$ is the function computed by the neural network after training at the $i^{th}$ epoch. This is formalized in the Occam Pruning Algorithm below, and provably good per Lemma 1.  The algorithm refers to $\tau_i$ as the "control" loss, which is the expected test loss at epoch $i$ as estimated over a small fraction of the training samples held back during training.  Also, let $\tau_{\infty}$ denote the algorithm's expected test loss at termination or in the limit.
\newpage
\begin{lemma}
If the initial pruning rate $\lambda_0>0$ is such that $\tau_3 <\tau_{2}<\tau_{1}$, and the pruning algorithm is well-conditioned, (a) $\tau_{i}$ converges monotonically; (b) the expected size of the network in terms of its dimension decreases by a factor no less than $e^{-\lambda_0(2+(\tau_{\infty}-\tau_2)/(\tau_2 - \tau_1)}$; and (c) if the Lipschitz constant of the pruning algorithm $c<1/2$ over the interval $[0,\lambda_0/(\tau_1/\tau_2-1)]$,  $\tau_{\infty}$ is less than the expected test loss of the underlying training algorithm.
\end{lemma}
\begin{proof}
\textbf{Part (a) }We proceed by induction. Basis: $\tau_3 <\tau_{2}<\tau_{1}$; Induction: assume $\tau_i <\tau_{i-1}<\tau_{i-2}$ and show that $\tau_{i+1} <\tau_i$.    Now
\begin{align}\label{lambda}
\lambda_i = \frac{\tau_i - \tau_{i-1}}{\tau_{i-1} - \tau_{i-2}}\lambda_{i-1}  = 
\frac{(\tau_i - \tau_{i-1})(\tau_{i-1} - \tau_{i-2})...(\tau_{3} - \tau_{2})}
{(\tau_{i-1} - \tau_{i-2})...(\tau_{3} - \tau_{2})(\tau_{2} - \tau_{1})}\lambda_0
=   \frac{\tau_{i} - \tau_{i-1}}{\tau_{2} - \tau_{1}}\lambda_{0} 
\end{align}
Also, $\lambda_0>0$, per the inductive hypothesis $\tau_i - \tau_{i-1}<0$, and per the basis $\tau_2-\tau_1<0$. It follows that $\lambda_i>0$.  Next, per Equation \ref{agnostic_theta}, after the training step at each epoch, the expected test loss is asymptotically of the following form, where $m =|S|$ is the number of samples and  $C$ is a positive constant.
\begin{align}\label{tau_i}
\tau_{i}= L(f_{i},P) = \min_{h\in F_{i-1}}L(h,P)+ 
C  (dim(F_{i-1})/m)^{0.5}
\end{align}
After pruning, per Equation \ref{prune_error}
\begin{align*}
\min_{h\in F_{i}} L(h,S) -\min_{h\in F_{i-1}} L(h,S) \leq  c \lambda _{i}[( L(f_{i},S)   - \min_{h\in F_{i-1}} L(h,S))]
\end{align*}
Taking expectations on both sides of the above, we get
\begin{align*}
\min_{h\in F_{i}} L(h,P) -\min_{h\in F_{i-1}} L(h,P) \leq  
c \lambda_{i} [( L(f_i,P)   - \min_{h\in F_{i-1}} L(h,P))]
\end{align*}
Substituting Equation \ref{tau_i} in the right-hand-side above, we get
\begin{align}\label{delta_best_loss}
\min_{h\in F_{i}} L(h,P) -\min_{h\in F_{i-1}} L(h,P) \leq  
c \lambda_{i} C(dim(F_{i-1})/m)^{0.5}
\end{align}
Incrementing epoch index $i$ in Equation \ref{tau_i} we get
\begin{align}\label{tau_iplus1}
     \tau_{i+1}=L(f_{i+1},P) &= \min_{h\in F_{i}}L(h,P)+ 
    C  (dim(F_{i})/m)^{0.5}
\end{align}
Substituting Equations \ref{prune_dim} and \ref{delta_best_loss} in the above, we get
\begin{align*}
     \tau_{i+1} &\leq \min_{h\in F_{i-1}}L(h,P)+c \lambda_{i} C(dim(F_{i-1})/m)^{0.5}  +
    C (dim(F_{i-1})/m)^{0.5} (1-\lambda_i)^{0.5}
\end{align*}
Substituting the Taylor series inequality $(1-\lambda_i)^{0.5} \leq (1-\lambda_i/2)$ and rearranging, we get
\begin{align*}
     \tau_{i+1} \leq [\min_{h\in F_{i-1}}L(h,P)+ C (dim(F_{i-1})/m)^{0.5}] + 
    \lambda_{i} C(c-1/2)(dim(F_{i-1})/m)^{0.5}
\end{align*}
Substituting Equation \ref{tau_i} in the above and rearranging, we get
\begin{align}\label{delta_tau_1}
     \tau_{i+1} - \tau_{i} \leq  \lambda_{i} C(c-1/2)(dim(F_{i-1})/m)^{0.5}
\end{align}
Likewise, we obtain
\begin{align}\label{delta_tau_2}
 \tau_{i} - \tau_{i-1} \leq \lambda_{i-1}C (c-1/2) (dim(F_{i-2})/m)^{0.5}
\end{align}
Substituting $\lambda_i = \lambda_{i-1} (\tau_i - \tau_{i-1})/(\tau_{i-1} - \tau_{i-2})$ per the algorithm in Equation \ref{delta_tau_1} we get
\begin{align*}
 \tau_{i+1} \leq \tau_{i} 
+ \lambda_{i-1}\frac{\tau_{i}- \tau_{i-1}}{\tau_{i-1}-\tau_{i-2}}
 C(c-1/2)(dim(F_{i-1})/m)^{0.5}
\end{align*}
Substituting Equation \ref{delta_tau_2} in the above, 
\begin{align}
 \tau_{i+1} \leq \tau_{i} 
+ \frac{\lambda_{i-1}^2C^2(c-1/2)^2}{\tau_{i-1}-\tau_{i-2}}
dim(F_{i-1})/m)^{0.5}dim(F_{i-2})/m)^{0.5}
 \label{prune_bound_3}
\end{align}
Earlier we showed that $\lambda_{i-1}>0$. Now $dim(F_{i-2}) >0$ else $\tau_{i-1} = \tau_{i-2}$ and the algorithm would have terminated. Likewise $dim(F_{i-1})>0$; and by the inductive assumption $\tau_{i-1}<\tau_{i-2}$. Hence the second term on the right-hand-side above is strictly negative implying $\tau_{i+1} < \tau_i$.  And per the inductive basis, $\tau_{3}< \tau_{2} <\tau_1$,  it follows that $\tau_{i+1}< \tau_{i}$.   

\textbf{Part (b)}  The final dimension of the network as a multiple of the initial dimension is as follows, using the inequality $1-x \leq e^{-x}$. 
\begin{align*}
\prod_{k=1}^{\infty} (1-\lambda_{k}) = 
(1 - \lambda_0)^2  \prod_{k=3}^{\infty} (1-\lambda_{k}) 
\leq e^{- 2\lambda_0 - \sum_{k=3}^{\infty} \lambda_{k}}
\end{align*}
Using Equation \ref{lambda}, we can write
\begin{align*} 
 2\lambda_0 + \sum_{k=3}^{\infty} \lambda_{k} = 2\lambda_0 + \sum_{k=3}^{\infty}\frac{\tau_{k} - \tau_{k-1}}{\tau_{2} - \tau_{1}}\lambda_{0} 
 = 2\lambda_0 + \frac{\tau_{\infty} - \tau_{2}}{\tau_{2} - \tau_{1}}\lambda_{0} 
 = \lambda_0 \left (2 +\frac{\tau_{\infty} - \tau_{2}}{\tau_{2} - \tau_{1}}\right).
\end{align*}
Hence the final dimension of the network is at most $e^{-\lambda_0(2+(\tau_{\infty}-\tau_2)/(\tau_2 - \tau_1)}$ the initial dimension.

\textbf{Part (c)} Returning to Equation \ref{tau_i} and recursively invoking Equations \ref{prune_dim} and \ref{prune_error}, we get
\begin{align}
\tau_{i} &= L(f_{i},P) = \min_{h\in F_{i-1}}L(h,P)+ C (dim(F_{i-1})/m)^{0.5}\notag\\
  &\leq \min_{h\in F_{i-2}}L(h,P)+  \lambda_{i-1
} C(c-1/2)(dim(F_{i-2})/m)^{0.5} +
  C (dim(F_{i-2})/m)^{0.5}\notag\\
  &...\notag\\
    &\leq  \min_{h\in F_{0}}L(h,P)+
    \sum_{k=1}^{i-1} \lambda_{k} C(c-1/2)(dim(F_{k-1})/m)^{0.5} +
C(dim(F_{0})/m)^{0.5} \label{inf_bound_1}
\end{align}
Since $\tau_i< \tau_{i-1}...\tau_1$ for all $i$, Equation \ref{lambda} implies that $\lambda_i < \lambda_0\tau_2/(\tau_1-\tau_2)$.  Therefore if $c<1/2$ over the interval $[0,\lambda_0/(\tau_1/\tau_2-1)]$, we have
\begin{align*}
C(c-1/2)\sum_{k=2}^{i-1} \lambda_{k}(dim(F_{k-1})/m)^{0.5} 
&\leq - C|1/2-c|\lambda_{0} (dim(F_{0})/m)^{0.5}
\end{align*}
Substituting $F_0 = F$ and combining the above with Equation \ref{inf_bound_1} we get
\begin{align}
\tau_{\infty} = \lim_{i\rightarrow \infty}\tau_{i}  \leq [\min_{h\in F}L(h,P)+C(dim(F)/m)^{0.5}] - C|1/2-c| \lambda_{0} (dim(F)/m)^{0.5}
\label{inf_bound_2}
\end{align}
Noting that the first term on the right of the above inequality is the expected test loss of the training algorithm without pruning, and that the second term is strictly negative, yields the result.
\end{proof}
The additional hyperparameter for the initial pruning rate $\lambda_0$ is a limitation of our algorithm.  If the underlying training algorithm is not optimal, an arbitrarily small $\lambda_0$ can be chosen to satisfy the condition of the lemma.  However, larger values of $\lambda_0$ improve performance per Equation \ref{inf_bound_2}.  Hence, as with all adaptive training, e.g., \citet{kb14}, the choice of the initial pruning rate plays an important role.  

The algorithm refers to $\tau_i$ as the control loss, which is the test loss over a small fraction of the training samples held back during training.  In our experiments of the next section, using the training loss for control performed just as well as using the loss on a holdback subset of training samples.  In practice, we cannot assert that popular pruning algorithms are well-conditioned per our definition above. Nevertheless, we use Occam Pruning to combine gradient descent training with the popular magnitude pruning into "Occam Gradient Descent," where at each pruning step, the smallest multiplicative weights by magnitude are clamped to zero, leaving the bias weights untouched.  Specifically, for pruning step size $\lambda$, multiplicative weights smaller in absolute value than the $\lambda$-quantile of each layer are clamped to zero, where the $\lambda$-quantile of a distribution is the value $q$ such that $\lambda$ is the mass of the distribution below $q$. For example, for $\lambda=0.5, q$ is the median.  Pruning can be based on other measures of relative importance amongst the weights; and other measures of the dimension of a network, e.g. the norm of the weights as in regularization, Bartlett (1996).  
\begin{figure}
    \centering
    \includegraphics[width=0.75\linewidth]{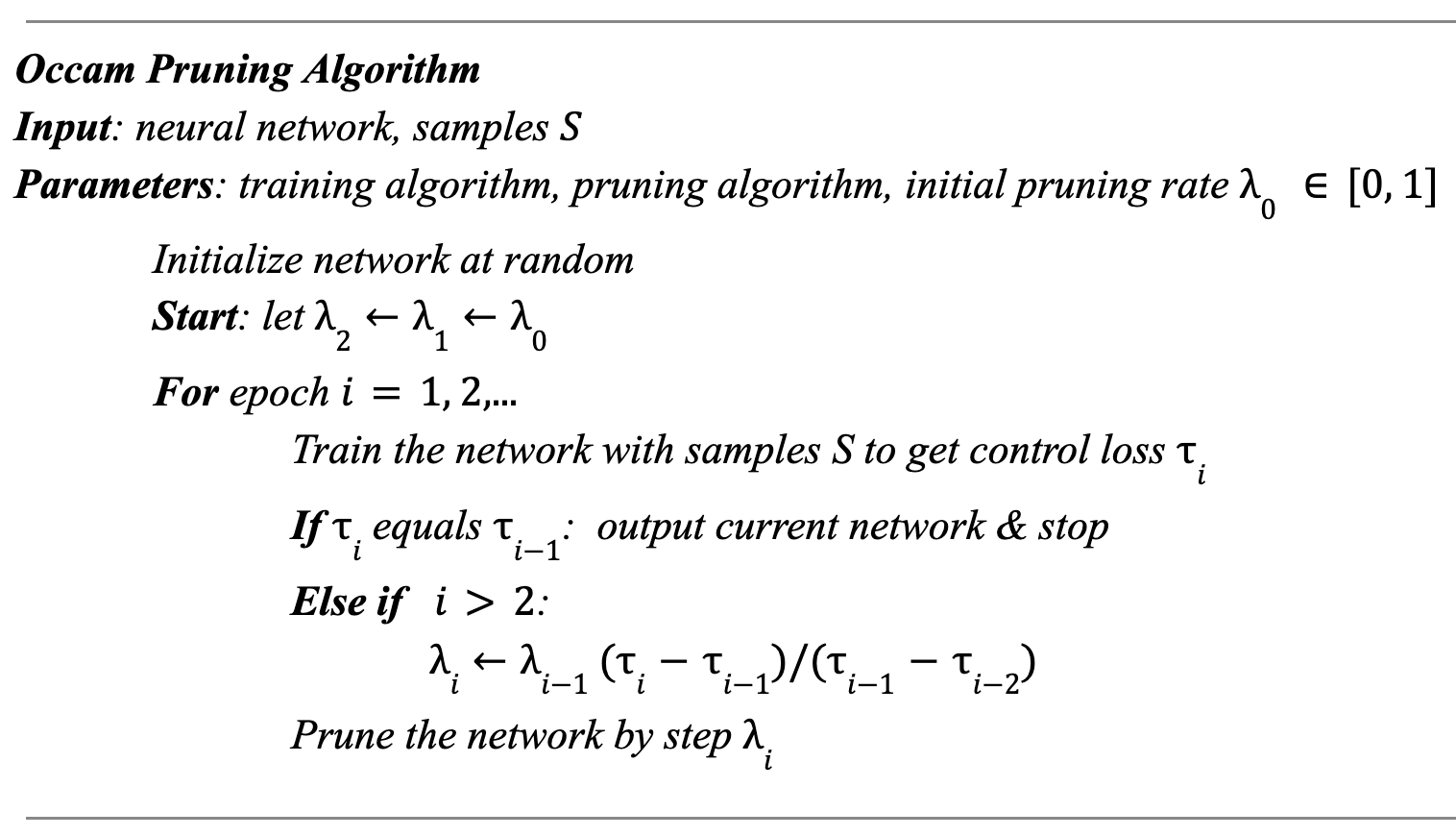}
\end{figure}
In practice, we found that stability is improved by constraining $\lambda$ to a positive interval, e.g. $[\lambda_0/10,\lambda_0 ]$. We also found that on large training sets,  the pruning step in the algorithm can be performed at fractional epochs. Furthermore, several variants of the algorithm are possible such as restoring a small random fraction of the weights that were set to zero previously.
\begin{figure}
    \centering
    \includegraphics[width=0.45\linewidth]{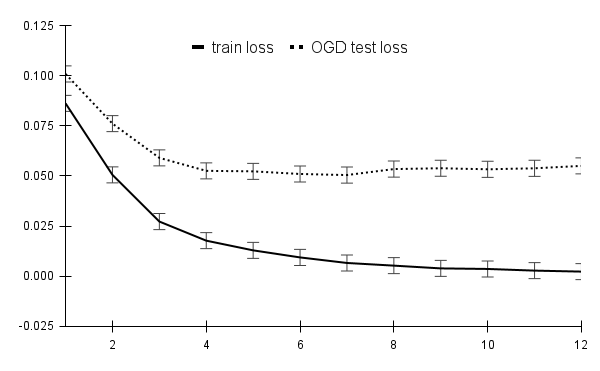}
    \includegraphics[width=0.45\linewidth]{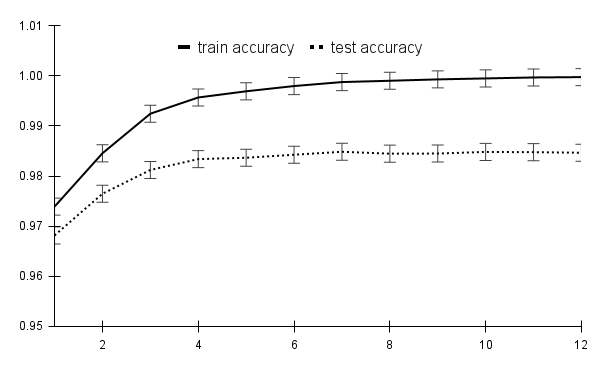}
    \caption{Occam Gradient Descent loss \& accuracy; MNIST; average of ten runs; $\lambda=0.4$}
    \label{fig:ogd_loss}
\end{figure}
Fig. \ref{fig:ogd_loss} shows the loss and accuracy with Occam Gradient Descent on the network of Fig. \ref{fig:MNIST_network} for MNIST.  Comparing Fig. \ref{fig:ogd_loss} with Fig. \ref{fig:gd_loss}, it is clear that the algorithm resists overfitting to substantially improve test loss and accuracy.  Furthermore, since the model size decreases with epochs, the total computation cost of training is reduced as noted in the next section.  If training is continued for many epochs after the test loss levels off, eventually the model gets so small that the conditions of Lemma 1 no longer hold, At that point, the test loss explodes.

In brief, Occam Gradient Descent interleaves adaptive reduction of model size to minimize the generalization error, with gradient descent on model weights to minimize training loss. In contrast, traditional gradient descent greedily minimizes training loss without regard to the generalization error.

\section{Experimental Results}
\label{Experimental_results}

Unless specified otherwise, all experiments in this paper used magnitude pruning of the multiplicative weights and gradient descent with the Adam optimizer, running on an M1-Mac laptop.   Our limited computational resources restricted us from testing the largest models.

Table \ref{tab:MNIST_results} compares the performance of the Occam Gradient Descent algorithm on the MNIST network of Fig.1 over twelve epochs, across ten runs.   For each algorithm, the table shows the statistics averaged at the epochs with the minimum test loss for each run. We remind the reader that the average at the minimum test loss is distinct from the minimum of the average test loss of Fig. \ref{fig:gd_loss} and Fig. \ref{fig:ogd_loss}. The first row shows the statistics for Gradient Descent. The second row shows the statistics for the Occam Gradient Descent algorithm at an initial pruning rate $\lambda_0 =0.4$ and $10\%$ holdback for the control loss. For this algorithm, at the minimum test loss, the average size of the network is ~$\approx 23\%$ of the original in terms of the number of non-zero weights. The projected compute cost is the average cumulative cost of the epochs across the shrinking network. The third row shows the statistics for Occam Gradient Descent without any holdback, but using the training loss for control. The last row shows the performance of conventional post-train pruning: 6 epochs of gradient descent followed by pruning to target size of $\approx 21\%$, and then retraining for 6 epochs.  The 6 epochs are the rounded average number of epochs at which traditional gradient descent minimized test loss, per the “Compute” column of the first row in the table.  It is evident that Occam Gradient Descent outperforms the other approaches at a lower computational cost, lower test loss, and smaller model size.

\begin{table}
\small
    \caption{Performance on MNIST network of Fig.\ref{fig:MNIST_network}; 12 epochs; 10 run average at best test loss} 
    \label{tab:MNIST_results}
    \begin{tabular}{ccccccc}
        \toprule
        Algorithm & training loss & Train Acc. & Test Loss & Test Acc. & Size & Compute \\
        \midrule
        Gradient Descent & 0.0187 &99.4\% &0.066 &98.1\% &100\% &5.9 \\
        \cmidrule(r){1-1}
        Occam Gradient Descent \\ (10\% holdback loss, $\lambda_0=0.4$) & 0.0066 & 99.9\% & 0.05 & 98.5\% & 23\% &3.1 \\
         \cmidrule(r){1-1}
        Occam Gradient Descent \\ (training loss, $\lambda_0=0.4$) & 0.008 & 99.9\% & 0.049 & 98.5\% & 21\% & 2.8 \\
         \cmidrule(r){1-1}
        Conventional post-train pruning & 0.0024 & 99.9\% & 0.056 & 98.5\% & 21\% & 7.1 \\
        \bottomrule
    \end{tabular}
\end{table}

Fig. \ref{fig:CIFAR_network} shows an oversized off-the shelf network for the CIFAR10 dataset combining both linear units and convolutional units. Fig. \ref{fig:CIFAR_results} shows the corresponding results comparing the performance of gradient descent and Occam Gradient Descent with training loss control. As for MNIST in Table \ref{tab:MNIST_results}, 10\% holdback loss control showed similar performance.  Fig. \ref{fig:CIFAR_network} also shows the fractional size of the model by epoch. It is clear that Occam Gradient Descent resists overfitting and outperforms gradient descent, even while reducing the network to a small fraction of its original size.
\begin{figure}
    \centering
    \includegraphics[width=0.75\linewidth]{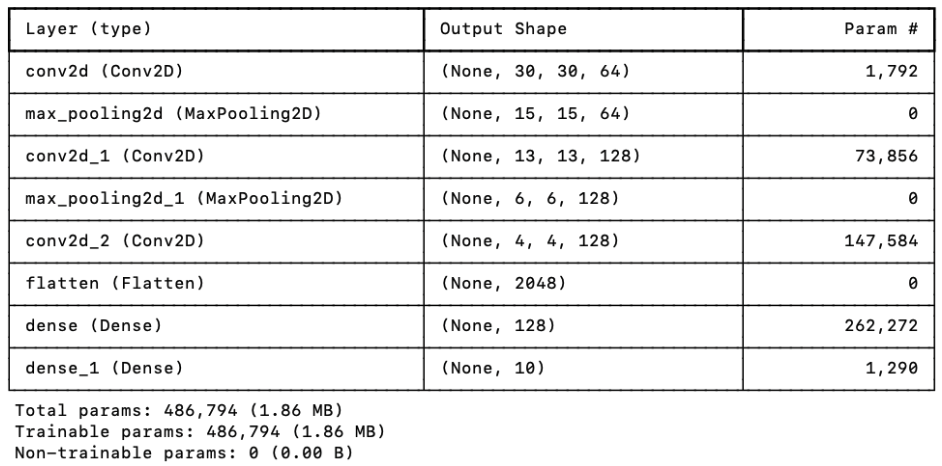}
    \caption{Oversized off-the shelf convolutional network for CIFAR10}
    \label{fig:CIFAR_network}
\end{figure}
\begin{figure}
    \centering
    \includegraphics[width=0.6\linewidth]{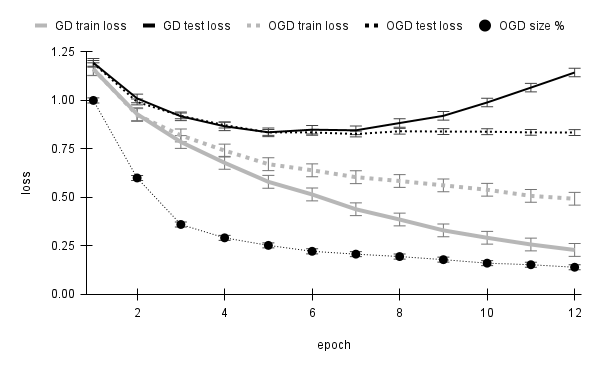}
    \caption{Gradient Descent (GD) \& Occam Gradient Descent (OGD);
 CIFAR10; training loss control; average of 10 runs; $\lambda_0=0.4$}
    \label{fig:CIFAR_results}
\end{figure}

\begin{figure}
    \centering
    \includegraphics[width=0.75\linewidth]{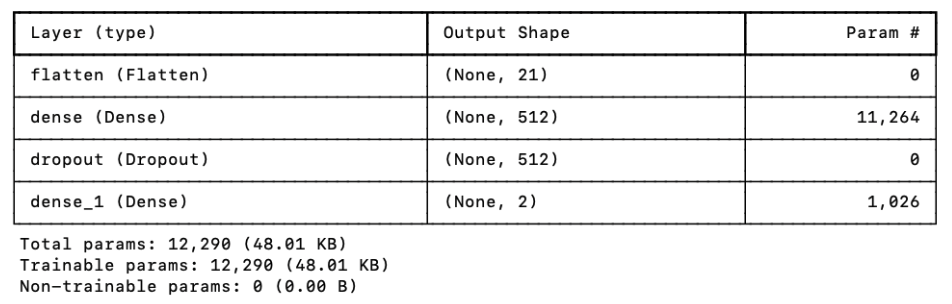}
    \caption{Dense linear network of 512 units for classification of tabular data with 21 input features}
    \label{fig:tabular_network}
\end{figure}

We now consider the classification of tabular data where an unknown feature must be predicted from a set of known features, a frequent application of machine learning. For example, given a set of patient vitals, diagnose the disease. Such problems are commonly addressed via Boosted Trees and Random Forests, see for example, \citet{sutton2005classification}, \citet{bs16}.  While Boosted Trees and Random Forests naturally resist overfitting, they typically create large models that scale up with the size of the data set.  In contrast, deep learning models are relatively compact but subject to overfitting when trained via gradient descent. Since Occam Gradient Descent addresses the overfitting limitation of deep learning networks, we test its applicability to classifying tabular data. Specifically, we compare the performance of Random Forests against a simple neural network with 512 dense linear units trained on tabular data sets. Fig. \ref{fig:tabular_network} shows the network for a data set with 21 input features.

Table \ref{tab:tabular_results} compares the performance of Random Forests and neural networks trained on a range of binary tabular classification data sets from the UC Irvine repository (https://archive.ics.uci.edu/datasets).  Each data set was randomly split into a training set comprising 75\% of the samples and a test set of the remaining samples, the split being fixed across runs. For Random Forests, we used the default settings in TensorFlow.  On each data set, neural networks of the form of Fig. \ref{fig:tabular_network} were trained for 12 epochs with Gradient Descent (GD Neural Network), and Occam Gradient Descent (OGD Neural Network) with pruning rate  $\lambda_0=0.4$ using training loss for control. Table \ref{tab:tabular_results} reports averages at the final epoch across ten runs. Fig. \ref{fig:tabular_compare} is a visual summary of Table \ref{tab:tabular_results}. In brief, compared to Random Forests on average across the data sets: neural networks trained with Gradient Descent are $\approx 16\%$ smaller at $\approx 11\%$ better cross-entropy loss; while neural networks trained with Occam Gradient Descent are $\approx 80\% $ smaller at $\approx 20\%$ better cross-entropy loss.

\begin{table}
\tiny
    \centering
    \caption{Tabular Classification: GD \& OGD Neural Networks; Random Forests (10 run average)}
    \label{tab:tabular_results}    
    \begin{tabular}{ccccccccl}
 & & & \multicolumn{2}{c}{Random Forest}& \multicolumn{2}{c}{GD Network}& \multicolumn{2}{c}{OGD Network}\\
         Dataset&  Samples&  Features&  Size (nodes)& Test Loss & Size (wts) & Test Loss & Size(wts) &Test Loss\\
         \midrule
         Census Income&  48,842&  14&  544,120&  0.429&  8,706&  0.311&  1,745&0.308\\
         Breast Cancer&  569&  30&  5,872&  0.236&  16,898&  0.103&  3,082&0.119\\
         Heart Disease&  303&  13&  11,754&  0.367&  8,194&  0.368&  1,893&0.374\\
 Credit Card Default& 30,000& 23& 546,372& 0.521& 13,314& 0.441& 2,448&0.434\\
 Room Occupancy& 10,129& 18& 7,630& 0.003& 10,754& 0.005& 2,707&0.003\\
         CDC Diabetes &  253,680&  21&  2,909,684&  0.540&  12,290&  0.315&  2,239&0.313\\
         \bottomrule
    \end{tabular}

\end{table}

\begin{figure}
    \centering
    \includegraphics[width=0.45\linewidth]{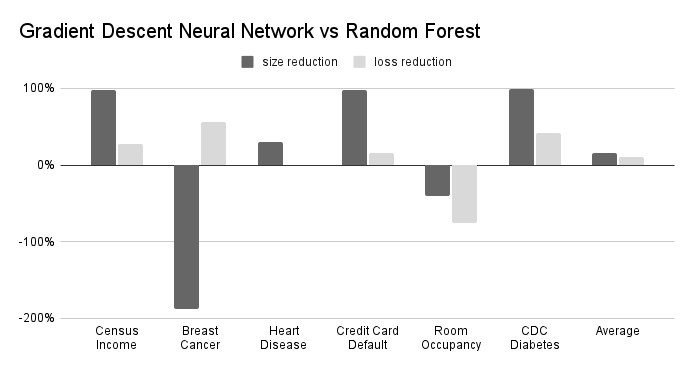}
    \includegraphics[width=0.45\linewidth]{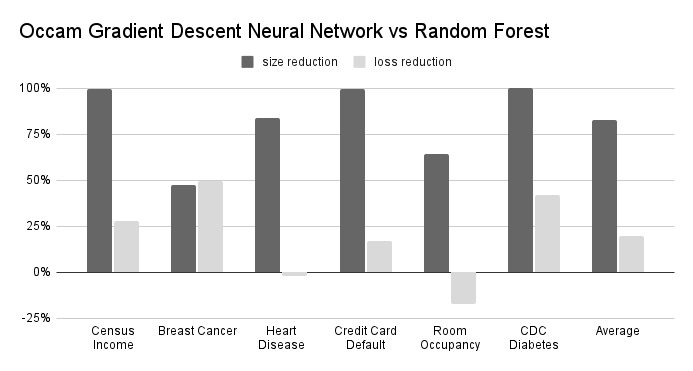}
    \caption{Size and loss reduction on tabular data sets}
    \label{fig:tabular_compare}
\end{figure}

Next, we apply Occam Gradient Descent to natural language transformer models.  Specifically, the open source nanoGPT model (https://github.com/karpathy/nanoGPT/blob/master/README.md) of Table \ref{tab:nanogpt} with dropout = 0.2 trained on Shakespeare’s works, with a training set of ~1M tokens, and a test set of $\approx 100K$ tokens. Fig. \ref{fig:OGD_transformer} shows the test loss for Gradient Descent and the test loss and model size for Occam Gradient Descent against training epochs.  Using training loss control, the initial pruning rate $\lambda_0=0.4$, and the contraction step is applied at intervals of 0.03 epochs. Under gradient descent training, the model overfits and the test loss is a minimum of 1.4587 at 0.18 epochs, rising with further training.  Under Occam Gradient Descent, both the test loss and the model size continue to improve with training. The test loss under Occam Gradient Descent surpasses that of Gradient Descent at 0.264 epochs, at which point the model size is $20\%$ of its original size in terms of the number of non-zero weights.  Reflective of the declining model size, the compute required by Occam Gradient descent for 0.264 epochs is ~60\% of the compute effort required by Gradient Descent to achieve its minimum loss at 0.18 epochs.  We note that this experiment uses the AdamW optimizer and that OGD performs on top of the built-in regularization.

\begin{table}
 \centering
 \caption{nanoGPT LLM with $\approx 11M$ parameters on Shakespeare’s works}
 \label{tab:nanogpt}
    \begin{tabular}{cccc}
    \toprule
         Layers&  Heads&  Context Length& Embed Dimension\\
         \midrule
         6&  6&  128& 384\\
          \bottomrule
    \end{tabular}
\end{table}

\begin{figure}
    \centering
    \includegraphics[width=0.8\linewidth]{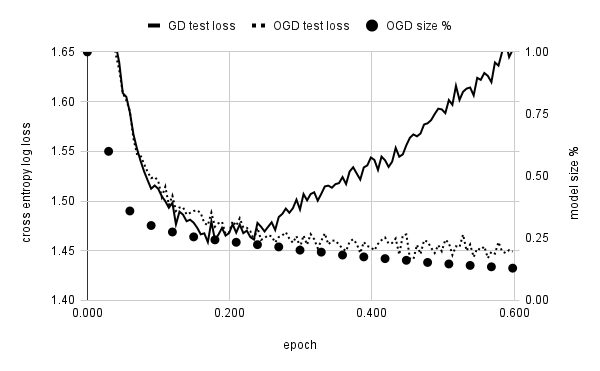}
    \caption{Gradient Descent (GD) \& Occam Gradient Descent (OGD) on LLM}
    \label{fig:OGD_transformer}
\end{figure}
\section{Summary}
Deep learning neural network models must be large enough to adapt to their problem domain, while small enough to avoid overfitting training data during gradient descent.  To balance these competing demands, overprovisioned deep learning models such as transformers are trained for a single epoch on large data sets, and hence inefficient with both computing resources and training data.  In response to these inefficiencies, we derive a provably good algorithm that combines training and pruning to simultaneously optimize efficiency and accuracy,  identifying conditions that resist overfitting and reduce model size while outperforming the underlying training algorithm. We then use the algorithm to combine gradient descent with magnitude pruning into "Occam Gradient Descent." With respect to loss, compute and model size  our algorithm outperforms traditional gradient descent on a broad range of data sets.  

\newpage
\bibliography{bibtex}
\bibliographystyle{iclr2026_conference}
\end{document}